\documentclass[preprint,1p]{elsarticle}
\usepackage[colorlinks,linkcolor=blue,anchorcolor=blue,citecolor=blue]{hyperref}
\usepackage{graphicx}%
\usepackage{subcaption}
\usepackage{multirow}%
\usepackage{amsmath,amssymb,amsfonts}%
\usepackage{amsthm}%
\usepackage{mathrsfs}%
\usepackage{xcolor}%
\usepackage{textcomp}%
\usepackage{manyfoot}%
\usepackage{booktabs}%
\usepackage{algpseudocode}%
\usepackage{listings}%
\usepackage[linesnumbered,ruled]{algorithm2e}

\usepackage[normalem]{ulem}
\useunder{\uline}{\ul}{}
\biboptions{numbers,sort&compress}
\newtheorem{definition}{Definition}

\newtheorem{theorem}{Theorem}

\usepackage{url}

\begin{document}
\begin{frontmatter}
\title{Outlier detection in mixed-attribute data: a semi-supervised approach with fuzzy approximations and relative entropy}

\author[inst1]{Baiyang Chen}                    
\author[inst1]{Zhong Yuan}                      
\author[inst2]{Zheng Liu}                       
\author[inst1,inst2]{Dezhong Peng\corref{cor1}}\ead{pengdz@scu.edu.cn}
\author[inst1]{Yongxiang Li}                    
\author[inst1]{Chang Liu}                       
\author[inst3]{Guiduo Duan}                     

\cortext[cor1]{Corresponding author: Dezhong Peng}
\affiliation[inst1]{organization={College of Computer Science, Sichuan University}, city={Chengdu}, postcode={610065}, country={China}}
\affiliation[inst2]{organization={Sichuan National Innovation New Vision UHD Video Technology Co., Ltd.}, city={Chengdu}, postcode={610095}, country={China}}
\affiliation[inst3]{organization={School of Computer Science and Engineering, University of Electronic Science and Technology of China}, city={Chengdu}, postcode={611731}, country={China}}

\begin{abstract}
    Outlier detection is a critical task in data mining, aimed at identifying objects that significantly deviate from the norm. Semi-supervised methods improve detection performance by leveraging partially labeled data but typically overlook the uncertainty and heterogeneity of real-world mixed-attribute data. This paper introduces a semi-supervised outlier detection method, namely fuzzy rough sets-based outlier detection (FROD), to effectively handle these challenges. Specifically, we first utilize a small subset of labeled data to construct fuzzy decision systems, through which we introduce the attribute classification accuracy based on fuzzy approximations to evaluate the contribution of attribute sets in outlier detection. Unlabeled data is then used to compute fuzzy relative entropy, which provides a characterization of outliers from the perspective of uncertainty. Finally, we develop the detection algorithm by combining attribute classification accuracy with fuzzy relative entropy. Experimental results on 16 public datasets show that FROD is comparable with or better than leading detection algorithms. All datasets and source codes are accessible at https://github.com/ChenBaiyang/FROD.
    \footnote{This manuscript is the accepted author version of a paper published by Elsevier. The final published version is available at \url{https://doi.org/10.1016/j.ijar.2025.109373
    }.}
    \end{abstract}
    

\end{frontmatter}


\section{Introduction}
    Outlier detection (OD), also known as anomaly detection, aims to identify objects or observations that deviate from the norm or expected behaviors. Outliers can have a significant impact on data analysis, as they may lead to inaccurate or misleading conclusions. 
    Moreover, anomalous data points may have important implications for various systems. For example, abnormal data in healthcare can affect patient monitoring, diagnostics, and treatment decisions. Unusual vital signs, abnormal laboratory results, or unexpected medical imaging findings can prompt further investigation, alter treatment plans, or trigger public health responses.   
    Therefore, OD is critical in data preprocessing and data mining and has various applications, including medical diagnosis \cite{Fernando2020Neural}, fraud detection \cite{pourhabibi2020fraud}, network intrusion detection \cite{dey2019amachine}, dangerous driving detection \cite{Liu2022Symbolic}, etc.
    
    As outliers are typically rare events and labeled data are often difficult to collect, most detection algorithms \cite{Liu2008IForest, Jiang2015Outlier, Ruff2018DeepSVDD, Li2022ECOD, almardeny2022ROD, Li2022DCROD} are designed to identify outliers without relying on any labeled data. Meanwhile, a small number of labels can be obtained and used to improve OD performance, thus reducing the false positive rate \cite{Estiri2019}. Consequently, many semi-supervised detection approaches \cite{2014OE, Pang2018RAMODO, Zhao2018XGBOD, Pang2019DevNet, Pang2019PReNet, Ruff2020DeepSAD, Huang2021ESAD, Zhou2022FEAWAD}, which leverage partially labeled data, have emerged as promising solutions to improve detection performance with minimal supervision.
    However, these methods often fail to account for two critical challenges in real-world applications: uncertainty and heterogeneity. Uncertainty arises from incomplete, noisy, or imprecise data, which are common in real-world scenarios and can obscure the distinction between normal and anomalous instances. Ignoring uncertainty can lead to unreliable or biased detection results. 
    Heterogeneity, on the other hand, refers to the presence of mixed-attribute data \cite{Zhang2016mixeddata, Chen2024COD}, where attributes include both numeric and categorical values or other complex types. Most existing methods simplify this complexity by transforming non-numeric data into numerical formats, which may lead to a loss or distortion of the original information. As a result, these methods struggle to capture the intricate relationships within heterogeneous datasets, reducing their effectiveness in complex, real-world datasets.
     
    To address these challenges, the fuzzy rough set (FRS) theory \cite{dubois1990rough} offers a flexible framework for knowledge classification with uncertainty by modeling a concept with fuzzy approximations. 
    It enables us to classify objects even if their values do not precisely match the criteria for a particular knowledge class. Thus, FRS can be used to identify and classify instances that are difficult to categorize as either normal or anomalous \cite{yuan2022FRGOD, Chen2024MFIOD}.
    Furthermore, various fuzzy membership functions are available to represent the membership degree of data points that have multiple types of values or representations, such as hybrid fuzzy similarity \cite{Yuan2021FIEOD} and kernelized functions \cite{Hu2011kernelized}. 
    This allows for directly processing mixed-attribute data without data type transformation, thus retaining more valuable information for knowledge discovery and data mining.

    The motivation for this work arises from the observation that outliers, defined as objects that significantly deviate from the majority of a dataset, introduce additional possible outcomes to a system. Consequently, the presence of outliers increases the uncertainty or randomness within the system.
    In this context, Shannon's information theory, as a mathematical model for measuring uncertainty, can be used to characterize outliers \cite{Jiang2010IE, Yuan2021FIEOD}.
    Specifically, we leverage the unlabeled data to compute an outlier factor for each data point based on fuzzy information entropy, which generalizes Shannon's entropy to the fuzzy case. Then, a minimal number of labeled data are used to construct fuzzy decision systems, through which we introduce the attribute classification accuracy derived from fuzzy approximations to evaluate the contribution of an attribute set in outlier detection. Finally, given a group of attribute sets, if the resulting outlier factors of an object are always high and the associated attribute classification accuracies are relatively large, then we may regard the object as an outlier in the dataset. 

    With the above-mentioned considerations, this paper proposes a semi-supervised outlier detection method, namely fuzzy rough sets-based outlier detection (FROD), for mixed-attribute data. The main contributions of this paper are outlined as follows.
    \begin{itemize}[]
        \color{blue}
        \item This paper proposes a fuzzy rough sets-based outlier detection algorithm in a semi-supervised manner.        
        \item We introduce attribute classification accuracy to assess their contribution with minimal labeled data.
        \item We employ fuzzy relative entropy to characterize outliers from the perspective of uncertainty.
        \item Experiments on 16 public datasets demonstrate that FROD performs better than or equal to leading detectors.
    \end{itemize}

{\color{blue}The remainder of this paper is organized as follows: In the next section, we review the related works. Section 3 provides the necessary preliminaries. In Section 4, we present our proposed method, followed by the experimental results and analysis in Section 5. Finally, Section 6 concludes this paper.}

\section{Related works}
    Outlier detection techniques are broadly categorized into unsupervised, semi-supervised, and supervised detectors according to their degree of reliance on labeled data. Supervised detection methods are often impractical due to their dependence on large amounts of labeled data, which can be difficult and costly to obtain. Therefore, we focus on reviewing unsupervised and semi-supervised approaches in this section.
    
\subsection{Unsupervised outlier detection (UOD)}
    UOD algorithms aim to identify outliers according to the inherent characteristics of data and do not rely on any labeled data. 
    These methods typically make assumptions about the distribution or expected behavior of the data and specify instances that deviate significantly from these assumptions. 
    UOD methods mainly include the following branches:
\begin{itemize}
    \item Statistical-based methods assume that normal data are generated by a probabilistic model, and those found in the low-probability regions are considered outliers. These methods are generally divided into parametric and non-parametric approaches. Parametric methods assume a distribution for the data and estimate model parameters from known data, often using Gaussian models \cite{Yang2009GMM} or regression models \cite{zhang2013advancements, satman2013new}. In contrast, non-parametric methods do not assume a specific distribution while typically learn distribution directly from the input data. Histogram-based techniques are widely used \cite{gebski2007efficient, Goldstein2012HBOS}, but selecting an appropriate bin size can lead to high false negative rates. Kernel Density Estimation (KDE) can improve this, as shown by \citet{latecki2007outlier}, although it struggles with high-dimensional data. \citet{zhang2018adaptive} explore an adaptive KDE method for nonlinear systems, while more recently, \citet{Li2022ECOD} introduce a non-parametric model based on empirical cumulative distributions for efficient outlier detection (ECOD).

    \item Distance-based methods identify outliers as objects that are significantly distant from the majority of other points \cite{Knorr1998Algorithms}. While simple and widely used, early distance-based methods are sensitive to parameter selection and computational inefficiency. \citet{Ramaswamy2000Efficient} improve efficiency by ranking objects based on their distance to k-nearest neighbors (kNN), identifying top-ranked ones as outliers. \citet{angiulli2002fast, angiulli2005outlier} use space-filling curves to accelerate computations, while \citet{bay2003mining} employ randomization for near-linear time performance. \citet{ghoting2008fast} propose a faster algorithm with logarithmic complexity in data size. To address sensitivity to the k parameter, \citet{Wang2015Afast} develop a minimum spanning tree-based method for detecting both global and local outliers. Recent works often leverage deep learning techniques. The Deep Support Vector Data Description (DeepSVDD) model \cite{Ruff2018DeepSVDD} projects data into a hypersphere that encloses normal data points, while the Learnable Unified Neighbourhood-based Anomaly Ranking (LUNAR) model \cite{Goodge2022LUNAR} uses graph neural networks to assess outliers through a message passing framework.

    \item Density-based methods assume that outliers are located in regions of low data density.  Local Outlier Factor (LOF) \cite{Breunig2000LOF} is the first work, which measures the local density deviation of a point to its neighbors. 
    Following LOF, various density-based methods have emerged.
    For example, Connectivity-based Outlier Factor (COF) \cite{Tang2002COF} improves LOF by considering the neighbors' connectivity in addition to their density.
    Relative Density Factor (RDF) \cite{Ren2004RDF} compares an object’s density to its neighbors, while Influenced Outlierness (INFLO) \cite{Jin2006INFLO} address k-nearest neighbor sensitivity by considering both nearest and reverse neighbors.
    Local Outlier Probabilities (LoOP) \cite{Kriege2009LoOP} combine density-based scoring with probabilistic methods.
    Kernel density estimation-based methods further improve robustness, with Gao et al. \cite{gao2011rkof} proposing the Robust Kernel-based Outlier Factor (RKOF) and Tang and He \cite{tang2017alocal} developing the Relative Density-based Outlier Score (RDOS). 
    More recently, \citet{Li2022DCROD} introduce the density changing rate (DCROD) for outlier detection.
    \citet{Yuan2023WFRDA} propose the definition of fuzzy-rough density using fuzzy rough sets for mixed-attribute data (WFRDA).
    
    \item Clustering-based methods assume that outliers do not belong to any cluster or belong to small, sparse clusters. For instance, He et al. \cite{He2003CBLOF} introduce Cluster-Based Local Outlier Factor (CBLOF) to identify these outliers. \citet{Huang2017ROCF} propose Relative Outlier Cluster Factor (ROCF) that first clusters the dataset by constructing mutual neighborhood graphs and then creates a decision graph to detect outliers and outlier clusters. These algorithms typically focus on finding clusters, often treating outliers as noise that can be ignored or tolerated, which can lead to performance bottlenecks. Therefore, some researchers handle clustering and anomaly detection at the same time. For example, \citet{hautamaki2005improving} propose a clustering algorithm that removes outliers while simultaneously identifying clusters and outliers, improving centroid estimates of generated distributions. \citet{gan2017k} enhance the k-Means algorithm by adding an extra cluster to accommodate all outliers, thereby providing both data clustering and outlier detection. 
    Similarly, \citet{Liu2021COR} present the Clustering with Outlier Removal (COR) algorithm that employs Holoentropy to assess cluster compactness while identifying outliers during the clustering process. 
    \end{itemize}
    
    It's worth noting that the success of UOD algorithms heavily relies on the appropriateness of the assumptions made by the method \cite{han2022ADbench}, as well as appropriate parameter tuning, such as choosing the proper value for the number of nearest neighbors. 
    
\subsection{Semi-supervised outlier detection (SSOD)}
    SSOD methods involve using a combination of labeled and unlabeled data to detect outliers. 
    Labeled data often provide information about what constitutes an outlier, while unlabeled data may contain valuable information about the underlying structure, patterns, and distribution of the data, which can aid in identifying outliers. 
    Early SSOD approaches \cite{gao2006SSOD, Xue2010FRSSOD, Deng2016SSOD-AFW} rely on clustering techniques like K-means, combining limited labeled data with unlabeled data and applying penalty mechanisms to improve detection accuracy. As data complexity grows, distance-based unsupervised detectors are adapted to incorporate labeled data. For instance, \citet{Daneshpazhouh2014EODSP} introduce an entropy-based method to identify outliers by calculating distances between normal points and known outliers. Similarly, \citet{Ienco2016SAnDCat} assume distinct feature combinations between normal and abnormal instances and detect outliers in nominal data by analyzing distance differences.
    Machine learning-based methods \cite{Zhao2018XGBOD, Pang2018RAMODO, Pang2019DevNet} enhance feature extraction and nonlinear relationship modeling through techniques such as representation learning, ensemble learning, and neural networks. Specifically, 
    \citet{Pang2019DevNet} propose an end-to-end framework that uses a nonlinear mapping to convert raw features into anomaly scores. These scores are then compared to a Gaussian prior distribution, and a contrastive loss is used to minimize deviation for normal samples.
    \citet{Ruff2020DeepSAD} extend the DeepSVDD method with neural networks, incorporating labeled outliers to improve detection.
    Recently, \citet{Zhou2022FEAWAD} use a deep autoencoder for semi-supervised reconstruction of normal samples, where reconstructed features, differences, and errors yield anomaly scores, thereby increasing detection accuracy and generalization ability. Additionally, \citet{Pang2019PReNet} introduce relation neural networks and ordinal regression to learn pairwise sample relationships, 
    moving away from traditional assumptions about anomaly score distributions and enhancing model generalization on unlabeled data. As SSOD evolves, deep learning frameworks are becoming increasingly sophisticated and diverse, particularly for high-dimensional, complex data.

{\color{blue}
\subsection{Outlier detection with Fuzzy rough sets}
Fuzzy rough set (FRS) theory extends classical rough set theory by incorporating fuzziness, allowing for the approximation of uncertain concepts through indiscernibility relations among objects. This flexibility has led to its increasing adoption in outlier detection, where it facilitates the identification of irregular patterns across a variety of data types \cite{Chen2010Neighborhood, YUAN2018243, wang2021outlier, yuan2022FRGOD, Chen2024MFIOD}. 
For instance, \citet{Chen2010Neighborhood} develope a neighborhood-based outlier detection algorithm, while \citet{YUAN2018243} introduce a framework that utilizes neighborhood information entropy for detecting outliers in mixed-attribute datasets. Building on this, \citet{wang2021outlier} enhance this approach  with a weighting network model.
More recently, fuzzy rough granule-based detection methods for mixed data have gained attention, such as those proposed by \citet{yuan2022FRGOD} and \citet{Chen2024MFIOD}. However, a limitation of these existing methods is their inability to incorporate labeled data to enhance detection performance. In this work, we build upon fuzzy rough set-based techniques by leveraging a minimal number of labeled data, which helps in reducing the false positive rate  and improving the overall accuracy of outlier detection.}
    
\section{Preliminaries}
    In fuzzy rough sets, an information table $(O, A)$ with a set of objects $O$ and a set of attributes $A$ is expressed as a fuzzy information system. When the attribute set $A$ is divided into $C \cup \{d\}$, where $C$ represents the conditional attributes and $d$ denotes the decision attribute, and $C\cap \{d\}=\emptyset$, this fuzzy information system is also called a fuzzy decision system (FDS).
    The decision attribute $d$ is the target variable that we aim to predict or classify, while the conditional attributes $C$ are the features used to make predictions. The fuzzy decision system leverages the relationships between $C$ and $d$ to derive decision rules, aiding in decision-making processes under uncertainty.

    Given a FDS $(O, C \cup \{d\})$ with $n$ objects, each subset $B\subseteq C$ induces a fuzzy similarity relation $\widetilde{B}$ which is usually stored in a matrix $M_{\widetilde{B}}=\left(r_{ij}^{\widetilde{B}}\right)_{n\times n}$, where $r_{ij}^{\widetilde{B}}=\widetilde{B}(o_i,o_j)$ represents the degree to which object $o_i$ has a relation $\widetilde{B}$ with object $o_j$. The tuple $(O,\widetilde{B})$ is referred to as a fuzzy approximation space \cite{dubois1990rough}, where the approximations of fuzzy sets can be derived.

\begin{definition}\label{def_fuzzy_appr}\cite{chen2011granular}
    Let $(O,\widetilde{B})$ be a fuzzy approximation space, and $\widetilde{X}$ is a fuzzy set on $O$, the lower and the upper approximation of $\widetilde{X}$ with respect to the fuzzy relation $\widetilde{B}$ are a pair of fuzzy sets whose membership function respectively are
    \begin{equation}\label{eq:fuzzy_appr}
    \begin{split}
        \widetilde{B}_*\widetilde{X}(o_i)&=\underset{o\in O}{\mathop{\inf }}\,\max \bigg\{1-\widetilde{B}(o_i,o), \widetilde{X}(o)\bigg\}\\
        \widetilde{B}^*\widetilde{X}(o_i)&=\underset{o\in O}{\mathop{\sup }}\,\min \bigg\{\widetilde{B}(o_i,o), \widetilde{X}(o)\bigg\}       
        \end{split}
        \end{equation}
        \end{definition}
    
    The fuzzy lower approximation $\widetilde{B}_*\widetilde{X}$ expresses the degree of objects certainly belonging to $\widetilde{X}$, and the fuzzy upper approximation $\widetilde{B}^*\widetilde{X}$ indicates the degree of elements possibly belonging to $\widetilde{X}$. Then, the fuzzy approximation accuracy is introduced to express the ``quality" of the two approximations.
    \begin{definition}\cite{yuan2022FRGOD}
        Given a fuzzy set $\widetilde{X}$ and a fuzzy similarity relation $\widetilde{B}$, the fuzzy approximation accuracy of $\widetilde{X}$ by $\widetilde{B}$ is
        \begin{equation}\label{eq:faa_x}
            \alpha_{\widetilde{B}}(\widetilde{X})=\frac{|{\widetilde{B}_*\widetilde{X}}|}
                                            {|{\widetilde{B}^*\widetilde{X}}|}.
            \end{equation}
            \end{definition}
            
    Given a fuzzy approximation space $(O, \widetilde{B})$, the fuzzy similarity relation $\widetilde{B}$ can induce a generalized fuzzy partition of $O$, i.e., a set of fuzzy similarity classes (also called knowledge classes). 
    \begin{definition}\label{def_partition}
        The fuzzy generalized partition of $O$ induced by a fuzzy similarity relation $\widetilde{B}$ is
        \begin{equation}
        O/{\widetilde{B}}=\left\{[o_i]_{\widetilde{B}}\right\}_{o_i\in O},
        \end{equation}
        where $[o_i]_{\widetilde{B}}=\left(r_{i1}^{\widetilde{B}}, r_{i2}^{\widetilde{B}}, \dots, r_{in}^{\widetilde{B}}\right)$ is a fuzzy similarity class centered at $o_i$. 
        \end{definition}
        
    The fuzzy similarity class $[o_i]_{\widetilde{B}}$ reflects the similarity of $o_i$ to all objects in $O$. 
    Obviously, $[o_i]_{\widetilde{B}}(o_j)=r^{\widetilde{B}}_{ij}$. If $r^{\widetilde{B}}_{ij}=1$, then it suggests that $o_j$ certainly belongs to $[o_i]_{\widetilde{B}}$; If $r^{\widetilde{B}}_{ij}=0$, then $o_j$ definitely does not belong to $[o_i]_{\widetilde{B}}$. The fuzzy cardinality of $[o_i]_{\widetilde{B}}$ is computed by $\big|[o_i]_{\widetilde{B}}\big|=\sum\limits_{j=1}^{n}{r^{\widetilde{B}}_{ij}}$.
    It is clear that $1\le \big|[o_i]_{\widetilde{B}}\big|\le n$.

\section{Methodology}
    This section presents {\color{blue}the proposed method} in detail. The overall detection procedure is illustrated in Fig. \ref{fig_FROD}.

\begin{figure*}[!h]
    \centering
    \includegraphics[width=0.95\textwidth]{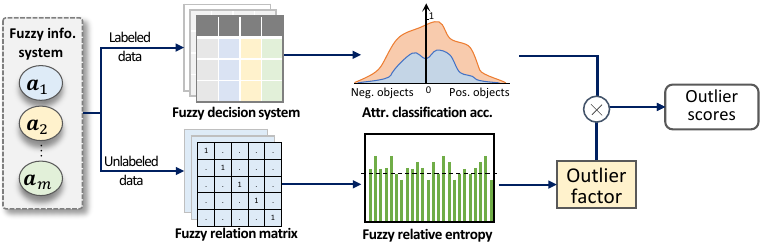}
    \caption{\color{blue}Overall structure of FROD. The method utilizes unlabeled data to calculate outlier factors based on fuzzy relative entropy. Then, a minimal number of labeled data is used to builed fuzzy decision systems, which enable to assess the contribution of attributes through attribute classification accuracy. Objects with consistently high outlier factors and significant attribute classification accuracy are identified as outliers.}
    \label{fig_FROD}
    \end{figure*}
    
\subsection{Attribute classification accuracy}\label{seq_aca}
    In semi-supervised settings, a small number of labeled objects are available. It is feasible to construct fuzzy decision systems (FDS) with features as conditional attributes and the class label as the decision attribute. 
    {\color{blue}Existing FRS-based approaches, such as fuzzy approximation accuracy \cite{yuan2022FRGOD} or classification error \cite{Wang2022Error}, evaluate the classification accuracy of conditional attributes. However, these methods do not fully leverage the available labeled information to exploit the potential of attribute sets in generating more separable classifications. To address this limitation, we propose a new measure of attribute classification accuracy that more effectively assesses the contribution of attribute sets to knowledge classification, thereby improving the reliability of outlier detection.}
    
    Let $O'=\{o_1,o_2,...\, o_{n'}\}$ be the set of objects with class labels, $C=\{c_1,c_2,..., c_{m}\}$ is the set of conditional attributes, where $c_k$ is the $k$-th attribute, and $d$ is the decision attribute, then $(O', C\cup \{d\})$ is a fuzzy decision system, based on which the fuzzy approximation accuracy \cite{yuan2022FRGOD} is defined.    
    \begin{definition}\label{def_faa}
        The fuzzy approximation accuracy (FAA) $\alpha_{C}(d)$ of the decision attribute $d$ by the conditional attribute $C$ is
        \begin{equation}\label{eq:faa}
            \alpha_{C}(d)=\frac{\left|\bigcup_{\widetilde{X}\in O'/\widetilde{d}}\widetilde{C}_*\widetilde{X}\right|}
            {\sum_{\widetilde{X}\in O'/\widetilde{d}}{\left|\widetilde{C}^*\widetilde{X}\right|}}.
            \end{equation}
            where $O'/\widetilde{d}$ is the fuzzy partition generated by the fuzzy similarity relation induced by $d$. 
            \end{definition}
            
    From the above definition, We can obtain that $0<\alpha_{C}(d)<=1$. 
    FAA describes the degrees to which all objects in $O'$ can be partitioned into the knowledge class induced by $d$ with respect to (w.r.t) $C$. 
    $\alpha_{C}(d)=1$ means that the knowledge class can be approximated precisely by the attribute set $C$. 
    In other words, FAA measures the capability of condition attributes to characterize the decision attribute. 
    Particularly, in outlier detection, all objects in $O'$ are categorized into two imbalanced crisp classes by the class label as 
    \begin{equation}\label{eq:partition_d}
            O'/d =\bigg\{[o^-]_{d}, [o^+]_{d} \bigg\},
            \end{equation}
    where $[o^-]_{d}$ and $[o^+]_{d}$ denote the normal class and the abnormal class (outliers) in $O'$, respectively. 
    Given an attribute subset $B\subseteq C$, the FAA of $B$ to $d$ is calculated as
    \begin{equation}\label{eq:faa_Bd}
        \alpha_{B}(d) 
        =\frac{\left|\widetilde{B}_*[o^-]_{d}\cup \widetilde{B}_*[o^+]_{d}\right|}{\left|\widetilde{B}^*[o^-]_{d}\right| + \left|\widetilde{B}^*[o^+]_{d}\right|}.
        \end{equation}

    As the numbers of positive instances and negative instances in outlier detection are generally imbalanced, we introduce a balancing parameter to control the weight of each class. 
\begin{definition}
    Let $(O', C \cup \{d\})$ be a FDS, $\forall B\subseteq C$, the attribute classification accuracy of the attribute set $B$ is
    \begin{equation}\label{eq_aca}
        \gamma_{B}=\frac{\left|\widetilde{B}_*[o^-]_{d}\right|}
        {\left|\widetilde{B}^*[o^-]_{d}\right|} + \beta \frac{\left|\widetilde{B}_*[o^+]_{d}\right|} {\left|\widetilde{B}^*[o^+]_{d}\right|},
        \end{equation}
    where $\beta$ is a control parameter to tune the weight of the two classes.
    \end{definition}
    
    This definition provides an effective measure of the contribution of an attribute set $B$ in knowledge classification, thus having the potential for outlier detection. {\color{blue}Figure~\ref{fig_aca} provides the illustration of the idea of attribute classification accuracy.}

    \begin{figure}
        \centering
        \includegraphics[width=6cm]{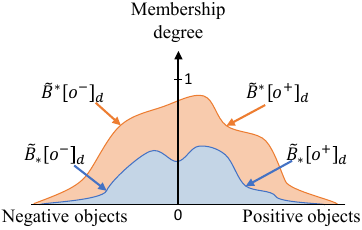}
        \caption{\color{blue}Illustration of attribute classification accuracy with respect to the attribute set $B$. The orange line indicates the membership degree of fuzzy upper approximation for each object, and the blue line denotes that of fuzzy lower approximation for each  object. The weighted propotion of blue area and orange area for negative objects and positive objects represents the attribute classification accuracy of $B$.}
        \label{fig_aca}
    \end{figure}
        
\subsection{Fuzzy relative entropy}
    This subsection first introduces fuzzy information entropy to measure information uncertainty within a fuzzy approximation space and then defines fuzzy relative entropy (FRE) to characterize outliers from unlabeled data. 
    Hu et al. \cite{hu2006fuzzy} generalized Shannon's information entropy to measure the information quantity implied in an attribute set or a fuzzy approximation space.
    \begin{definition}\label{def_FE}
        Given a fuzzy approximation space $(O, \widetilde{B})$, the fuzzy information entropy of the attribute set $B$ or the fuzzy similarity relation $\widetilde{B}$ is
        \begin{equation}\label{eq:FE}   
            FE({B})=FE(\widetilde{B})=-\frac{1}{|O|}\sum_{o_i \in O}\log_2{\frac{|[o_i]_{\widetilde{B}}|}{|O|}}.
        \end{equation}
        \end{definition}
    This definition forms a map: $FE: B \to \mathbb{R}^+$, which builds a foundation on which we can compare the uncertainty or randomness of a fuzzy approximation space.

    Since outliers are significantly different from the majority of a dataset. 
    The presence of outliers brings more possible outcomes to a system or, in other words, imposes more uncertainty or randomness on the system. 
    Therefore, uncertainty can be used to characterize outliers.
    To quantify this uncertainty, we introduce the fuzzy relative entropy (FRE) to measure the uncertainty associated with an object.
    \begin{definition}\label{def_FRE}       
        Given a fuzzy approximation space $(O, \widetilde{B})$, $\forall o_i\in O$ and $O_{\neg i}=O-\{o_i\}$, the fuzzy relative entropy $FRE_B(o_i)$ of $o_i$ on $B$ is
        \begin{equation}\label{eq:FRE}
            FRE_B(o_i)=\frac{FE_{\neg i}(B)}{FE(B)}+ \lambda,
        \end{equation}
        where $FE_{\neg i}(B)=-\frac{1}{|O_{\neg i}|}\sum_{o_j \in O_{\neg i}}\log_2{\frac{|[o_j]_{\widetilde{B}}|}{|O_{\neg i}|}}$ represents the fuzzy entropy after removing $o_i$ from $O$, $\lambda=\frac{1}{|O|}$ is a smoothing term.
        \end{definition}

    Clearly, the fuzzy relative entropy of $o_i$ is determined by comparing the relative change in the fuzzy entropy of the approximation space. 
    If the fuzzy entropy decreases significantly after removing $o_i$, then the uncertainty of $o_i$ w.r.t $B$ may be considered high; otherwise low. 
    Notably, the programming technique of memorization can be employed to avoid massive repeated calculations in computing FREs.
    Figure \ref{fig_ex} presents an illustrative example of FRE applied to a 2D dataset containing 100 normal points and 10 outliers. The results reveal that outliers exhibit significantly lower FRE values than normal points, highlighting the effectiveness of the proposed measure in distinguishing outliers. In the subsequent theorem, we will show that FRE inherently aligns with the intuitive understanding of uncertainty and outliers
    \begin{figure}[!h]
        \centering
        \begin{subfigure}[b]{0.44\linewidth}
            \includegraphics[height=4.6cm]{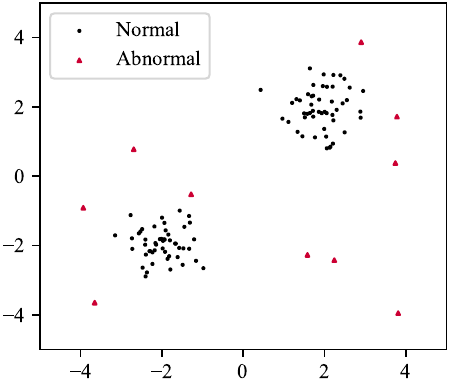}
            \caption{2D plot} \label{fig_ex_2D}
        \end{subfigure}
        \quad
        \begin{subfigure}[b]{0.44\linewidth}
            \includegraphics[height=4.6cm]{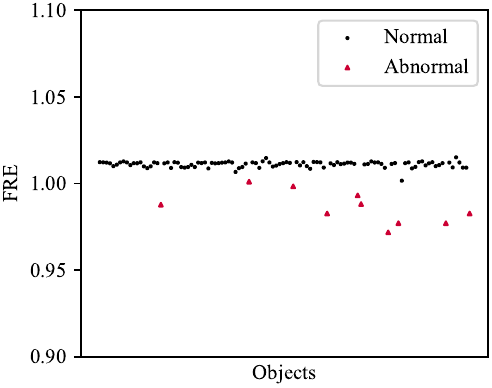}
            \caption{Fuzzy relative entropy} \label{fig_ex_FRE}
        \end{subfigure}
        \caption{An example with 100 normal points and 10 outliers: (a) 2D plot of the dataset. (b) Fuzzy relative entropy for each object.} \label{fig_ex}
        \end{figure}

\begin{theorem}\label{prop_p2}
    Let $O=O_1\cup \{o_k\}$ be the set of objects with the attributes $B$ such that $\forall o_i, o_j \in O_1$, $d_{ij}\leq \delta$ and $d_{ik}> \delta$. It holds that $FRE(o_i) > FRE(o_k)$.
    \end{theorem}
    
\begin{proof}
    Let $O_2=O-\{o_i\}$ be the subset obtained by removing the object $o_i$ from $O$, and let $O_3=O_1-\{o_i\}$ be the subset obtained by removing $o_i$ from $O_1$. Since $O=O_1\cup \{o_k\}$, we have $O_3=O_2-\{o_k\}$. By Definition \ref{def_FE}, we have
    \begin{align}\label{eq_p2_1}
    \begin{split}
        FE^{(O_1)}(B) & =FE^{\big(O_3\cup \{o_i\}\big)}(B)=-\frac{1}{|O_1|}\sum_{o_l\in O_3}\log_2\frac{1}{|O_1|}\left|[o_l]^{(O_1)}\right|-\frac{1}{|O_1|}\log_2\frac{1}{|O_1|}\left|[o_i]^{(O_1)}\right|,\\
        FE^{(O_2)}(B) & = FE^{\big(O_3\cup \{o_k\}\big)}(B)=-\frac{1}{|O_2|}\sum_{o_l\in O_3}\log_2\frac{1}{|O_2|}\left|[o_l]^{(O_2)}\right|-\frac{1}{|O_2|}\log_2\frac{1}{|O_2|}\left|[o_k]^{(O_2)}\right|.
        \end{split}
        \end{align}

    Given that $\forall o_i,o_j\in O_1, d_{ij} \leq \delta$ and $d_{ik}>\delta$, 
    we obtain $r_{ij}=1-\delta>0$ and $r_{ik}=0$. Therefore, $\forall o_l\in O_3$, we have $r_{lk}=0$, $r_{li}>0$ and
    \begin{align}
    \begin{split}
        \left|[o_l]^{(O_1)}\right| &= \sum_{o_j\in O_3}r_{lj}+r_{li},\\
        \left|[o_l]^{(O_2)}\right| &= \sum_{o_j\in O_3}r_{lj}+r_{lk}=\sum_{o_j\in O_3}r_{lj}.
        \end{split}
        \end{align}    

    Therefore, we obtain
    \begin{equation}
        \left|[o_l]^{(O_1)}\right|>\left|[o_l]^{(O_2)}\right|.
        \end{equation}
    
    Taking the logarithm and summing all $o_l$ in $O_3$, we get
    \begin{equation}\label{eq_p2_2}
        \sum_{o_l\in O_3}\log_2\frac{1}{|O_1|}\left|[o_l]^{(O_1)}\right| >
        \sum_{o_l\in O_3}\log_2\frac{1}{|O_2|}\left|[o_l]^{(O_2)}\right|.
        \end{equation}

    Since $\left|[o_i]^{(O_1)}\right|=\sum_{o_l\in O_3}r_{li}+1$ and $\left|[o_k]^{(O_2)}\right|=\sum_{o_l\in O_3}r_{lk}+1=1$, we have
    \begin{equation}\label{eq_p2_3}
        \left|[o_i]^{(O_1)}\right|>\left|[o_k]^{(O_2)}\right|.
        \end{equation}

    Combining Eqs.(\ref{eq_p2_1}), (\ref{eq_p2_2}) and (\ref{eq_p2_3}), we have $FE^{(O_1)}(B)<FE^{(O_2)}(B)$. Therefore
    \[FE^{(O)}_{\neg k}(B)=FE^{(O_1)}(B)<FE^{(O_2)}(B)=FE^{(O)}_{\neg i}(B).\]
    
    Therefore, $\frac{FE_{\neg i}^{(O)}(B)}{FE^{(O)}(B)}+\lambda>\frac{FE_{\neg k}^{(O)}(B)}{FE^{(O)}(B)}+\lambda$.
    By Definition \ref{def_FRE}, we establish $FRE(o_i) > FRE(o_k)$.
\end{proof}

Theorem \ref{prop_p2} provides a theoretical foundation for the behavior of FRE in identifying outliers by leveraging the relative changes in fuzzy entropy. Specifically, it demonstrates that the FRE of an object $o_i$ within a dense region $O_1$ is greater than that of a distant object $o_k$, where the dense region is characterized by objects with small inter-object distances ($d_{ij}\leq \delta$), and $o_k$ lies outside this region ($d_{ik} > \delta$).
This result reflects that objects within a sparse region contribute significantly to the change in fuzzy entropy. Consequently, removing such objects leads to a substantial reduction in entropy, resulting in lower FRE values.
The theorem underscores the effectiveness of FRE in quantifying uncertainty and characterizing outliers, establishing a robust theoretical basis for its application in outlier detection.

\subsection{Detection algorithm}
    In this subsection, we first introduce data representation with a hybrid fuzzy membership function and then define the fuzzy relative entropy to compute the outlier factor for every instance. 
    Subsequently, we combine attribute classification accuracy with these outlier factors to calculate the outlier scores for objects in the dataset. Finally, we design a thresholding method using labeled data to generate binary predictions on outliers.
    
\subsubsection{Data representation}
    As real-world data usually include mixed attributes, this work focuses on three types of values: numerical data, nominal data, and a mix of the two types.
    The numerical values are first transferred into $[0, 1]$ using the min-max normalization to maintain a uniform scale across all features.
    Then, we employ a hybrid fuzzy membership function \cite{yuan2022FRGOD} to represent the membership degree of objects that have multiple types of values. 
    
    Let $f_i^k$ be the value of the object $o_i$'s attribute $c_k$, and $d_{ij}^k={\left|f_i^k- f_j^k\right|}$ is the difference value between $o_i$ and $o_j$ on the attribute $c_k$, the membership degree $\widetilde{c}_k(o_i,o_j)$ of the fuzzy similarity relation $\widetilde{c}_k$ between $o_i$ and $o_j$ is determined by
    \begin{equation}\label{eq:rel_attr}
        \widetilde{c}_k(o_i,o_j)=\left\{
        \begin{array}{ll}
        \mathbb{I}_{f_i^k = f_j^k}, & \text{if } c_k \text{ is nominal};\\
        1-d_{ij}^k,   & \text{if } d_{ij}^k \le r^k, c_k \text{ is numerical};\\
        0,          & \text{if } d_{ij}^k > r^k, c_k \text{ is numerical};\\
        \end{array}\right.
        \end{equation}
    where $\mathbb{I}_{(\cdot)}$ represents an indicator function that returns 1 if the condition is met; otherwise, it returns 0. ${r^k}$ indicates the fuzzy radius for the attribute $c_k$ computed by
    \begin{equation}
        r^k=\delta \cdot ave(d_{ij}^k, c_k).
        \end{equation}
    where $ave(d_{ij}^k, c_k)$ is the average of the difference value $d_{ij}^k$ of each object pair $(o_i, o_j)$ in $O$ on the attribute $c_k$, and $\delta$ is an adjustable parameter. This setting ensures the fuzzy radius $r^k$ is adaptively determined based on the specific characteristics of each attribute, thus effectively handling variations and outliers within the data.

\subsubsection{Fuzzy relative entropy-based outlier factor}
    As the fuzzy relative entropy reflects the anomaly degree of an object, we may consider those objects in $O$ whose fuzzy relative entropies are always lower when compared with other objects and use the information contained within the relative entropy to identify outliers. Therefore, we can compute an outlier factor for each object based on fuzzy relative entropy from the unlabeled data. 
    \begin{definition}\label{def_OF}
         Let $(O,C\cup \{d\})$ be a FDS, $\forall {B}\subseteq C$, the outlier factor $OF(o_i)$ of the object $o_i\in O$ w.r.t the attribute set $B$ is
        \begin{equation}\label{eq:OF}
            OF_B(o_i)=W_B(o_i)\cdot FRE_B(o_i),
            \end{equation}
        where the weight function $W_B(o_i)=
                \sqrt{\frac{1}{|O|}\left|[o_i]_{\widetilde{B}}\right|}$.
            \end{definition}
        
    In the above definition, the lower $FRE_B(o_i)$ of $o_i$, the higher uncertainty of $o_i$ and the more likely $o_i$ is an outlier. 
    In addition, as the partition power of the attribute set $B$ is important for knowledge classification, the fuzzy approximation accuracy defined in the previous section is leveraged into the construction of outliers.

\subsubsection{Outlier scoring}
    Since an attribute set $B$ with a higher classification accuracy is more useful in outlier detection (see Section \ref{seq_aca}), this subsection utilizes attribute classification accuracy to guide the process of outlier scoring. The primary problem is to specify the attribute sets $B$. A common method \cite{Goldstein2012HBOS, Li2022ECOD, Yuan2023WFRDA} is to assume the independence of each attribute. Under this assumption, we assess the outlier degree for each object by combining the outlier factors from each attribute with their respective classification accuracy. This integration is obtained through a weighted summation as follows.

    \begin{definition}\label{def_OD}
         Let $(O,C\cup \{d\})$ be a FDS, $\forall o_i\in O$, the fuzzy relative entropy-based outlier degree of $o_i$ is
        \begin{equation}\label{eq:OD}
            OD(o_i)=1-\frac{1}{|C|}\sum\limits_{c_k\in C}{\gamma_{c_k} \cdot OF_{c_k}(o_i)}.
            \end{equation}
            \end{definition}
 
    In order to output binary prediction, we set an adaptive threshold parameter to detect outliers in unlabeled data.
    \begin{definition}\label{def_FROD}
        Let $\theta \in (0,1)$ be a real-valued threshold, $\forall o_i \in O$, if \textit{OD}$(o_i) > \theta$, then $o_i$ is regarded as an outlier.
        \end{definition}

    In the above definition, the optimal value of $\theta$ is adaptively determined. 
    A simple and practical way is to adopt the greatest outlier degree of normal instances in the labeled data.
    Given a set of labeled objects $O'=\{O^-,O^+\}$, where $O^-$ indicates normal data and $O^+$ represents outliers, the threshold $\theta$ is computed as
    \begin{equation}\label{eq_theta}
        \theta= \max_{o_i\in O^-} {OD}(o_i).
    \end{equation}

    Finally, the whole procedure of FROD is illustrated in Algorithm \ref{alg_FROD}. 
    Since FROD requires calculating the fuzzy relative entropy of each object on each conditional attribute, the time complexity for FROD is $O(mn)$, where $m$ denotes the number of attributes and $n$ represents the number of objects. As FROD stores a fuzzy relation matrix for each attribute, its space complexity is $O(mn^2)$.

\begin{algorithm}[!h]
    \caption{FROD algorithm}\label{alg_FROD}
    \LinesNumbered
    \KwIn{An FDS $(O, C \cup \{d\})$  with a small subset of labeled objects $O'=\{O^-, O^+\}\subseteq O$, where $O^-$ indicates normal data and $O^+$ represents outliers}
    \KwOut{Outlier set $OS$}
    $OS = \emptyset$\;
    \For {each $c_k\in C$}{
        Construct fuzzy relation ${\widetilde{c}_k'}$ {\color{blue}using the labeled data $O'$} by Eq.~(\ref{eq:rel_attr})\; 
        Compute attribute classification accuracy $\gamma_{c_k}$ by Eq.~(\ref{eq_aca})\;

        Construct fuzzy relation ${\widetilde{c}_k}$ {\color{blue}using the unlabeled data in $O-O'$}\;
        Compute fuzzy entropy $FE(c_k)$ by Eq.~(\ref{eq:FE})\;
        \For{{\color{blue}each} $o_i\in O-O'$}{
            Compute fuzzy relative entropy $FRE_{c_k}(o_i)$ by Eq.~(\ref{eq:FRE})\;
            Compute outlier factor $OF_{c_k}(o_i)$ by Eq.~(\ref{eq:OF})\;
            }
        }
    \For{each $o_i\in O-O'$}{
        Compute outlier degree $OD(o_i)$ by Eq.~(\ref{eq:OD})\;
        }
    Compute outlier threshold $\theta$ by Eq.~(\ref{eq_theta})\;
    \For{each $o_i\in O-O'$}{
        \If {$OD(o_i)>\theta$}{
            $OS = OS \cup \{o_i\}$\;
            }
        }
    \Return $OS$.
    \end{algorithm}
    
    \begin{table}[!h]
        \caption{An example of partially labeled data table}\label{Mixed_DT}
        \centering
        \begin{tabular}{ccccc}
        \toprule
        \multirow{2}{*}{Objects} & \multicolumn{3}{c}{Conditional attributes} & Decision attribute  \\
        \cmidrule{2-5}
                                 & $c_1$         & $c_2$        & $c_3$        & $d$                \\
        \midrule
        $o_1$                     & 0.53         & 7            & C            & Outlier            \\
        $o_2$                     & 0.48         & 8            & C            & Normal             \\
        $o_3$                     & 0.50         & 7            & B            & Normal             \\
        $o_4$                     & 0.48         & 8            & B            & Normal             \\
        $o_5$                     & 0.51         & 8            & B            & Normal             \\
        $o_6$                     & 0.52         & 7            & C            & Not available      \\
        $o_7$                     & 0.48         & 9            & A            & Not available      \\
        $o_8$                     & 0.47         & 8            & A            & Not available      \\
        $o_9$                     & 0.53         & 9            & A            & Not available      \\
        $o_{10}$                  & 0.48         & 9            & B            & Not available      \\
        \bottomrule
        \end{tabular}
        \end{table}
        
\subsection{Example}
    A concrete example is given in this part to show how the proposed algorithm works. Note that only the basic concepts are demonstrated here.
    A partially labeled data table is shown in Table \ref{Mixed_DT}, where $O=\{o_1,o_2,...,o_{10}\}$ is the set of objects, $C=\{c_1,c_2,c_3\}$ is the conditional attribute set containing a real-valued, an integer-valued and a categorical attribute, $d$ is the decision attribute with 5 labeled values and 5 unlabeled ones.

    We first apply min-max normalization to all numerical values and then use the labeled data (i.e., $O'=\{o_1,o_2,...,o_5\}$) to construct a fuzzy decision system. Let the fuzzy radius parameter be $\delta=1$, then we have the fuzzy radius $r^1=0.3467$, $r^2=0.24$, and the resulting fuzzy similarity relations induced by the conditional attributes are as\\
    \begin{small}
    $M'_{\widetilde{c}_1}=\left(\begin{array}{ccccc}
            1&0&0&0&0.667\\
            0&1&0.667&1&0\\
            0&0.667&1&0.667&0.833\\
            0&1&0.667&1&0\\
            0.667&0&0.833&0&1\\
      \end{array}\right)$,
    $M'_{\widetilde{c}_2}=\left(\begin{array}{ccccc}
            1&0&1&0&0\\
            0&1&0&1&1\\
            1&0&1&0&0\\
            0&1&0&1&1\\
            0&1&0&1&1\\
        \end{array}\right)$,
    $M'_{\widetilde{c}_3}=\left(\begin{array}{ccccc}
            1&1&0&0&0\\
            1&1&0&0&0\\
            0&0&1&1&1\\
            0&0&1&1&1\\
            0&0&1&1&1\\
        \end{array}\right)$.
    \end{small}
        
    The crisp equivalence class $d$ can be generated by the labeled decision attribute: $[o^-]_d=(0,1,1,1,1)$,$[o^+]_d=(1,0,0,0,0)$. Let $\beta=1$, the attribute classification accuracy of each conditional attribute is computed by Eq.~(\ref{eq_aca}) as 
  
        $\gamma_{c_1}=\frac{\left|\widetilde{c}_{1*}[o^-]_{d}\right|}{\left|\widetilde{c}_1^*[o^-]_{d}\right|}
        + \beta \frac{\left|\widetilde{c}_{1*}[o^+]_{d}\right|} {\left|\widetilde{c}_1^*[o^+]_{d}\right|}
        =\frac{0.333}{1.667} + \frac{3.333} {4.667}
        \approx 0.914$,
        
        $\gamma_{c_2}==\frac{0}{2}+\frac{3}{5}=0.6$, $\gamma_{c_3}=\frac{0}{2}+\frac{3}{5}=0.6$.

    Next, we use the unlabeled data (i.e., $O-O'=\{o_6,o_7,$ $...,o_{10}\}$) to construct the fuzzy similarity relations as\\
        \begin{small}
		$M_{\widetilde{c}_1}=\left(\begin{array}{ccccc}
                1&0&0&0.833&0\\
                0&1&0.833&0&1\\
                0&0.833&1&0&0.833\\
                0.833&0&0&1&0\\
                0&1&0.833&0&1\\
          \end{array}\right)$,
		$M_{\widetilde{c}_2}=\left(\begin{array}{ccccc}
                1&0&0&0&0\\
                0&1&0&1&1\\
                0&0&1&0&0\\
                0&1&0&1&1\\
                0&1&0&1&1\\
            \end{array}\right)$,
		$M_{\widetilde{c}_3}=\left(\begin{array}{ccccc}
                1&0&0&0&0\\
                0&1&1&1&0\\
                0&1&1&1&0\\
                0&1&1&1&0\\
                0&0&0&0&1\\
            \end{array}\right)$.
        \end{small}

    By Definition \ref{def_FE}, the fuzzy entropy is computed as\\
    $FE(c_1)=-\frac{1}{5}\big(\log_2{\frac{1.833}{5}}+\log_2{\frac{2.833}{5}}+\log_2{\frac{2.666}{5}}+\log_2{\frac{1.833}{5}}+\log_2{\frac{2.833}{5}}\big)\approx 1.088$. 
    Similarly, we have $FE_{\neg 6}(c_1)\approx 0.895$.
    
    By Definition \ref{def_FRE}, the fuzzy relative entropy of $o_6$ w.r.t $c_1$ is\\
    $FRE_{c_1}(o_6)=\frac{FE_{\neg 6}(c_1)}{FE(c_1)}+ \lambda=\frac{0.895}{1.088}+\frac{1}{5}\approx 1.023$.
    
    Then, we obtain the outlier factor of $o_6$ on $c_1$ as
    $OF_{c_1}\left(o_6\right) =W_{c_1}(o_6)\cdot FRE_{c_1}(o_6)=\sqrt{1.8333/5}\times 1.023=0.619$.
    Similarly, $OF_{c_2}\left(o_6\right)=OF_{c_3}\left(o_6\right) \approx 0.3542$.
    
    Therefore, the outlier score for each unlabeled object is computed as\\ 
    $OD(o_6)=1-\frac{1}{|C|}\sum_{c_k}{\gamma_{c_k} \cdot OF_{c_k}(o_6)}=1-\frac{1}{3}\big(0.914\times 0.619+0.6\times 0.3542+0.6\times 0.3542\big)\approx 0.670$.
    Similarly, we have $OD(o_7)\approx 0.316$, $OD(o_8)\approx 0.467$, $OD(o_9)\approx 0.410$, $OD(o_{10})\approx 0.446$.
    
    Comparing the outlier degrees of each object, it can be seen that the outlier degree of $o_6$ is significantly higher than the others. Let the threshold $\theta$ be 0.6, then the object $o_6$ is regarded as an outlier.

\section{Experiments}
    This section tries to empirically investigate the following questions:
    \begin{itemize}[\indent]
        \item \textbf{Detection performances in real-world scenarios.} Can FROD generalize from a limited number of labeled data to effectively detect outliers in real-world datasets?
        \item \textbf{Effectiveness with mixed-attribute data}. Can FROD make better use of mixed-attribute data?
        \item \textbf{Detection performances with various amounts of labeled data.} How effective is FROD in utilizing labeled data for guiding outliers detection?
        \item \textbf{Statistical significance analysis.} Whether there are statistically significant differences between FROD and comparison methods?
    \end{itemize}

\subsection{Experimental settings}
    We select 16 real-world datasets that cover various fields,s including healthcare, image, botany, etc., to assess the performances of detectors. The details of the datasets are presented in Table \ref{datasets}. For semi-supervised methods, we randomly select 1\% of labeled data to train the models and use the remaining to test the detectors. Following previous works \cite{han2022ADbench, Jiang2023WSAD}, we adapt the unsupervised algorithms for predicting the new instances, i.e., 1\% of labeled data are used to tune parameters, and the rest for testing. 
    {\color{blue}For the proposed method, the parameters $\delta$, $\beta$ were tuned using grid search with the labeled data which are available in the semi-supervised setting. Specifically, $\delta$ was varied from 0.1 to 3 in increments of 0.5, and $\beta$ was chosen from the set $\{0.01,0.1,1,10,100\}$. The combination of values that yielded the best performance was selected for final evaluation. }
    Each experiment is independently run 10 times, in each of which we employ stratified sampling to preserve a consistent proportion of outliers. Finally, the average results are reported.

\begin{table*}[!hbt]
    \centering \tabcolsep=2pt
    \caption{Details of the experimental datasets}\label{datasets}
    \begin{tabular}{ccccccc}
    \toprule
    No. & Dataset  & \# Object & \# Attribute & \# Outlier & Category  & Data type\\
    \midrule
    1  & Annthyroid \cite{han2022ADbench}  & 7200       & 6             & 534        & Healthcare  & Numerical   \\
    2  & Arrhythmia  \cite{Yuan2023WFRDA}  & 452        & 279           & 66         & Medical     & Mixed       \\
    3  & Cardiotocography\cite{han2022ADbench}& 2114    & 21            & 466        & Healthcare   & Numerical   \\
    4  & Ionosphere  \cite{han2022ADbench} & 351        & 32            & 126        & Oryctognosy & Numerical   \\
    5  & Mammography \cite{han2022ADbench} & 11183      & 6             & 260        & Healthcare  & Numerical   \\
    6  & Mushroom1   \cite{yuan2022FRGOD}  & 4429       & 22            & 221        & Botany      & Categorical \\
    7  & Mushroom2   \cite{Yuan2023WFRDA}  & 4781       & 22            & 573        & Botany      & Categorical \\
    8  & Musk        \cite{han2022ADbench} & 3062       & 166           & 97         & Chemistry   & Numerical   \\
    9  & Optdigits   \cite{han2022ADbench} & 5216       & 64            & 150        & Image       & Numerical   \\
    10 & PageBlocks  \cite{han2022ADbench} & 5393       & 10            & 510        & Document    & Numerical   \\
    11 & Pima        \cite{han2022ADbench} & 768        & 8             & 268        & Healthcare   & Numerical   \\
    12 & Satellite   \cite{han2022ADbench} & 6435       & 36            & 2036       & Astronautics & Numerical   \\
    13 & Satimage-2  \cite{han2022ADbench} & 5803       & 36            & 71         & Astronautics & Numerical   \\
    14 & Sick        \cite{Yuan2023WFRDA}  & 3613       & 29            & 72         & Medical      & Mixed       \\
    15 & SpamBase    \cite{han2022ADbench} & 4207       & 57            & 1679       & Document     & Numerical   \\
    16 & Thyroid     \cite{han2022ADbench} & 3772       & 6             & 93         & Healthcare   & Numerical  \\
    \bottomrule
    \end{tabular}
    \end{table*}
    
\subsubsection{Comparison methods}
    We choose the following 5 semi-supervised methods as well as 5 unsupervised approaches to compare with FROD.

\textbf{Unsupervised approaches:}
    \begin{itemize}
    \item IForest (2008) \cite{Liu2008IForest}: An ensemble model which constructs binary trees to isolate objects. The estimator parameter is selected in \{5, 10, 50, 100, 500\}.
    \item DeepSVDD (2018) \cite{Ruff2018DeepSVDD}: A deep learning-based detector that adopts one-class classification for outlier detection. The epoch number parameter is tuned from \{5, 10, 20, 50, 100, 200\}.
    \item ECOD (2022) \cite{Li2022ECOD}: A statistical method that employs the empirical cumulative distribution in a non-parametric fashion.
    \item MFGAD (2023) \cite{YUAN2023MFGAD}: An FRS-based method that introduces multi-fuzzy granules to compute outlier scores. The radius parameter is selected from 0.1 to 2.
    \item WFRDA (2023) \cite{Yuan2023WFRDA}: A FRS-based detector that defines fuzzy-rough density to detect outliers. The radius parameter is tuned from 0.1 to 2 with a stepsize of 0.1. 
    \end{itemize}

\textbf{Semi-supervised approaches:}
    \begin{itemize}
    \item REPEN (2018) \cite{Pang2018RAMODO}: A deep learning-based detector that learns feature representations by exploiting random distances. The margin in the loss function is set to 1000.
    \item DevNet (2019) \cite{Pang2019DevNet}: A deep learning-based detector that introduces the Gaussian prior and deviation loss to train models. The margin parameter is set to 5.
    \item DeepSAD (2020) \cite{Ruff2020DeepSAD}: A deep learning-based detector that extends DeepSVDD to leverage labeled data. The loss parameter $\eta$ is set to 1.
    \item FEAWAD (2022) \cite{Zhou2022FEAWAD}: A deep learning-based detector that employs autoencoders to predict outlier scores. The loss parameter $a$ is set to 5.
    \item PReNet (2023) \cite{Pang2019PReNet}: A deep learning-based detector that leverages pairwise sample relationships to identify outliers.
    \end{itemize}

\subsubsection{Evaluation metrics}
    We evaluate the comparison algorithms by two complementary metrics: AUC (Area Under the ROC Curve) and AP (Average Precision). 
    AUC emphasizes class separability, i.e., how well the algorithm is able to distinguish between different classes. While AP emphasizes the completeness of detection, i.e., how many relevant instances are retrieved by the detector. Note that AUC is insensitive to class imbalance, while AP is otherwise. 
    The two metrics are both computed based on the predicted outlier scores of objects. AUC summarizes a receiver operating characteristic (ROC) curve into a real number between 0 and 1. A higher AUC indicates a better performance.    
    Similarly, AP summarizes a precision-recall curve into a real number between 0 and 1. A greater AP value indicates a higher detection accuracy, and that more positive instances are retrieved.

\subsection{Experimental results and analysis}
\subsubsection{Detection performances in real-world scenarios}
    In real-world applications, there usually exists a large number of unlabeled data, and a small set of labeled data can be easily obtained. Therefore, we use only 1\% labeled data as the train set and the remaining 99\% as the test set to evaluate the detection approaches. 
    To ensure a fair comparison, we use the same labeled data as the semi-supervised models to adjust hyperparameters for the unsupervised algorithms, that is, using 1\% labeled data to test the unsupervised algorithms in a given range of parameters and selecting the one with the largest AUC score as the optimal parameter.

\begin{table*}[!h]
    \centering \tabcolsep=2pt
    \caption{AUC scores (\%) of detection algorithms (@1\% labeled data). The highest score is bolded, and the second one is underlined.}\label{AUC}
    \resizebox{\textwidth}{!}{
    \begin{tabular}{c|ccccc|cccccc}
    \toprule
    Dataset          & IForest     & DeepSVDD & ECOD           & MFGAD          & WFRDA          & DeepSAD & REPEN          & DevNet         & PReNet            & FEAWAD      & FROD            \\
    \midrule
    Annthyroid       & {\ul 87.54} & 71.79    & 78.89          & 72.03          & 65.85          & 62.47   & 62.48          & 68.78          & 73.89          & 80.78       & \textbf{96.30}  \\
    Arrhythmia       & 68.72       & 75.88    & 50.00          & 75.57          & \textbf{81.31} & 69.46   & 71.23          & 57.77          & 57.67          & 60.04       & {\ul 80.40}     \\
    Cardiotocography & 68.00       & 59.79    & 78.55          & 65.81          & 75.30          & 59.60   & 68.73          & 76.11          & {\ul 80.88}    & 79.96       & \textbf{81.07}  \\
    Ionosphere       & 61.74       & 77.85    & 72.84          & 47.92          & {\ul 78.52}    & 77.88   & 77.05          & 52.41          & 53.56          & 61.50       & \textbf{80.93}  \\
    Mammography      & 85.70       & 82.80    & \textbf{90.59} & 43.60          & 83.41          & 85.20   & 84.47          & 86.20          & 85.19          & 80.97       & {\ul 88.39}     \\
    Mushroom1        & 94.09       & 81.46    & 50.00          & \textbf{97.38} & {\ul 97.01}    & 86.59   & 89.31          & 88.93          & 89.75          & 93.01       & 95.06           \\
    Mushroom2        & 87.62       & 78.20    & 50.00          & 91.55          & 88.21          & 82.00   & 88.53          & \textbf{98.98} & {\ul 98.86}    & 95.70       & 98.50           \\
    Musk             & 89.36       & 81.13    & 95.59          & 98.92          & {\ul 99.98}    & 95.10   & 89.29          & 67.46          & 92.97          & 92.92       & \textbf{100.00} \\
    Optdigits        & 70.59       & 46.63    & 50.00          & 59.71          & 92.68          & 78.14   & 59.50          & 93.72          & \textbf{99.86} & {\ul 97.06} & 90.30           \\
    PageBlocks       & 87.39       & 87.08    & \textbf{91.40} & 39.44          & 84.88          & 48.52   & 55.80          & 75.70          & 70.56          & 81.11       & {\ul 90.79}     \\
    Pima             & 61.46       & 60.91    & 59.41          & 63.91          & \textbf{68.84} & 55.02   & 61.54          & 57.48          & 58.01          & 59.84       & {\ul 64.34}     \\
    Satellite        & 71.56       & 66.10    & 58.28          & 70.64          & 76.25          & 82.37   & 75.56          & 79.24          & \textbf{84.02} & {\ul 82.64} & 82.08           \\
    Satimage-2       & 96.32       & 96.15    & 96.55          & 96.22          & 96.93          & 81.70   & \textbf{99.14} & 81.57          & 96.84          & 81.83       & {\ul 97.82}     \\
    Sick             & 74.46       & 71.20    & 50.00          & 49.43          & {\ul 79.02}    & 68.40   & 74.92          & 62.28          & 66.44          & 69.76       & \textbf{86.43}  \\
    SpamBase         & 70.46       & 56.31    & 65.54          & 55.16          & 71.66          & 65.22   & 72.73          & 87.54          & \textbf{90.21} & 84.93       & {\ul 87.62}     \\
    Thyroid          & 96.75       & 92.89    & 97.70          & {\ul 98.55}    & 90.84          & 88.20   & 89.21          & 96.52          & 98.03          & 96.56       & \textbf{99.29}  \\
    \midrule
    Average          & 79.48       & 74.14    & 70.96          & 70.36          & {\ul 83.17}    & 74.12   & 76.22          & 76.92          & 81.05          & 81.16       & \textbf{88.71}  \\
    \bottomrule
    \end{tabular}}
    \end{table*}

    As seen from Table~\ref{AUC}, FROD performs best on 6 datasets, including Annthyroid, Cardiotocography, Ionosphere, Musk, Sick, and Thyroid. FROD also performs comparably to the best detector on 6 datasets, ranking second. Additionally, FROD obtains the highest average AUC over all 16 datasets. For instance, the average AUC of FROD is obviously higher than the best-unsupervised baseline WFRDA (5.54\% higher) and the best-semi-supervised competitor FEAWAD (81.16\% higher). This demonstrates FROD's effectiveness in detecting outliers with a very limited number of labeled data.

\begin{table*}[!h]
    \centering \tabcolsep=2pt
    \caption{AP scores (\%) of detection algorithms (@1\% labeled data). The highest score is bolded, and the second one is underlined.}\label{PR}
    \resizebox{\textwidth}{!}{
    \begin{tabular}{c|ccccc|cccccc}
    \toprule
    Dataset          & IForest & DeepSVDD       & ECOD  & MFGAD          & WFRDA          & DeepSAD & REPEN & DevNet         & PReNet            & FEAWAD      & FROD            \\
    \midrule
    Annthyroid       & 42.39   & 29.39          & 27.12 & 37.46          & 16.89          & 20.59   & 16.17 & 24.86          & 33.18          & {\ul 42.67} & \textbf{77.66}  \\
    Arrhythmia       & 29.70   & 46.07          & 14.54 & 35.27          & {\ul 46.43}    & 34.69   & 35.93 & 27.59          & 27.34          & 28.42       & \textbf{48.23}  \\
    Cardiotocography & 41.23   & 31.55          & 50.56 & 41.12          & 47.07          & 41.29   & 39.70 & 54.97          & \textbf{64.02} & {\ul 60.52} & 58.50           \\
    Ionosphere       & 46.98   & \textbf{71.05} & 64.53 & 47.48          & 66.79          & 67.57   & 66.62 & 46.61          & 48.26          & 55.09       & {\ul 70.88}     \\
    Mammography      & 19.45   & 18.81          & 43.31 & 6.24           & 9.39           & 24.13   & 14.76 & \textbf{46.50} & {\ul 44.08}    & 41.77       & 40.93           \\
    Mushroom1        & 61.45   & 32.83          & 4.99  & 64.44          & \textbf{89.77} & 28.28   & 25.52 & 83.50          & 85.90          & 83.63       & {\ul 89.63}     \\
    Mushroom2        & 50.90   & 30.77          & 11.98 & 44.84          & 67.43          & 38.55   & 38.42 & \textbf{96.23} & {\ul 95.11}    & 86.72       & 94.69           \\
    Musk             & 28.28   & 38.39          & 49.27 & 73.90          & {\ul 99.52}    & 83.25   & 49.93 & 15.41          & 81.58          & 70.42       & \textbf{100.00} \\
    Optdigits        & 5.32    & 2.80           & 2.87  & 3.45           & 32.84          & 15.82   & 4.94  & 48.24          & \textbf{97.44} & {\ul 73.38} & 37.82           \\
    PageBlocks       & 43.46   & {\ul 53.55}    & 52.03 & 16.49          & 36.62          & 20.51   & 17.84 & 42.95          & 41.40          & 48.00       & \textbf{58.82}  \\
    Pima             & 46.04   & 44.68          & 46.31 & {\ul 51.33}    & \textbf{53.57} & 43.64   & 47.43 & 45.71          & 46.38          & 48.44       & 49.72           \\
    Satellite        & 67.44   & 54.39          & 52.60 & 64.49          & 70.47          & 74.49   & 70.14 & 75.48          & \textbf{80.85} & {\ul 79.55} & 74.93           \\
    Satimage-2       & 63.48   & 75.49          & 65.61 & 79.33          & 82.20          & 61.57   & 81.27 & 32.02          & \textbf{90.92} & 67.50       & {\ul 83.38}     \\
    Sick             & 5.98    & 5.14           & 1.98  & 2.39           & 5.19           & 4.44    & 5.02  & 8.71           & 8.87           & {\ul 9.25}  & \textbf{10.21}  \\
    SpamBase         & 57.75   & 44.77          & 51.81 & 47.46          & 62.84          & 59.16   & 62.11 & 80.19          & {\ul 84.41}    & 76.80       & \textbf{86.31}  \\
    Thyroid          & 52.48   & 46.16          & 46.83 & \textbf{77.05} & 23.80          & 37.41   & 21.85 & 68.99          & {\ul 76.94}    & 69.88       & 75.98           \\
    \midrule
    Average          & 41.40   & 39.12          & 36.65 & 43.30          & 50.67          & 40.96   & 37.35 & 49.87          & {\ul 62.92}    & 58.88       & \textbf{66.11}  \\
    \bottomrule
    \end{tabular}}
    \end{table*}
    Table~\ref{PR} shows the results regarding AP, FROD achieved similar results to AUC. Namely, FROD performs the best on 6 datasets (Annthyroid, Arrhythmia, Musk, PageBlocks, Sick, and SpamBase), and ranks second on 3 datasets. The average score FROD obtained is much higher than the best-unsupervised method WFRDA (15.44\%) and the best-semisupervised competitor PReNet (3.19\%). The results confirm FROD's effectiveness in detecting outliers with a very limited number of labeled data in various scenarios.

    However, FROD does not perform the best in some cases, e.g., PReNet beats FROD on Optdigits and Satellite by 9.56\% and 1.94\%, respectively. Moreover, some unsupervised methods perform even better than semi-supervised ones with a low level of supervision (1\%), for example, the best-unsupervised method WFRDA (on average) achieves 4.5\% and 0.92\% higher AUC than FROD on Pima and Arrhythmia, respectively.
    The reason may be that FROD's assumption of information entropy-based outliers is not quite appropriate in situations involving unknown outlier types.
    {\color{blue} Furthermore, we observed that FROD achieved relatively low AP scores on the Sick (10.21) and Optdigits (37.82) datasets, as shown in Table~\ref{PR}. For the Sick dataset, all comparison methods had very low AP scores, indicating that this is a particularly difficult task. Nevertheless, our method outperforms all others on this dataset. The Optdigits dataset, being an image dataset with spatial dependencies among features, presents a challenge for FROD, as our method does not incorporate image-specific techniques. Consequently, It is not well suited for this data type.
    }

\begin{figure}[!hbt]
    \centering
    \includegraphics[width=8cm]{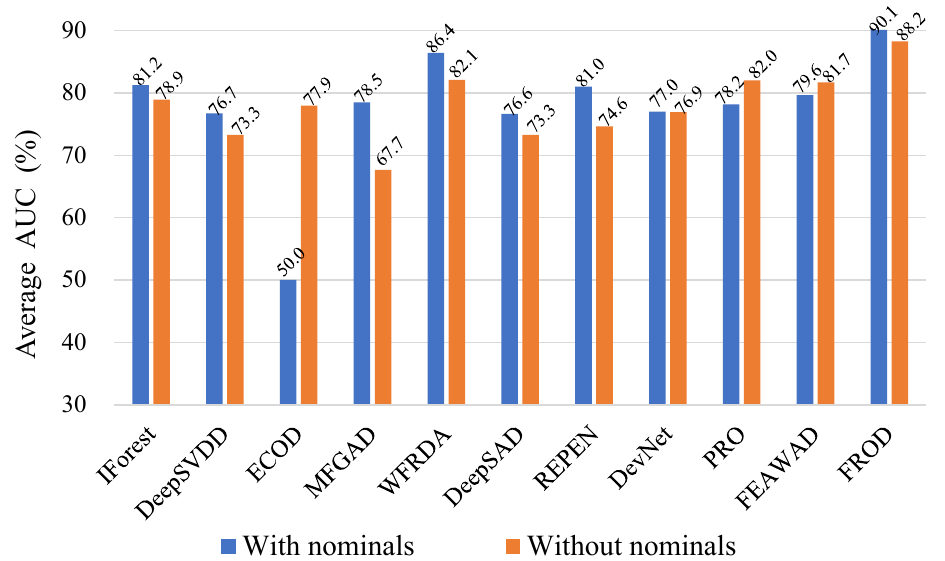}
    \caption{Average AUC (\%) on datasets with and without nominal attributes (@1\% labeled data)}\label{AUC-nominals}
    \end{figure}

\begin{figure}[!hbt]
    \centering
    \includegraphics[width=8cm]{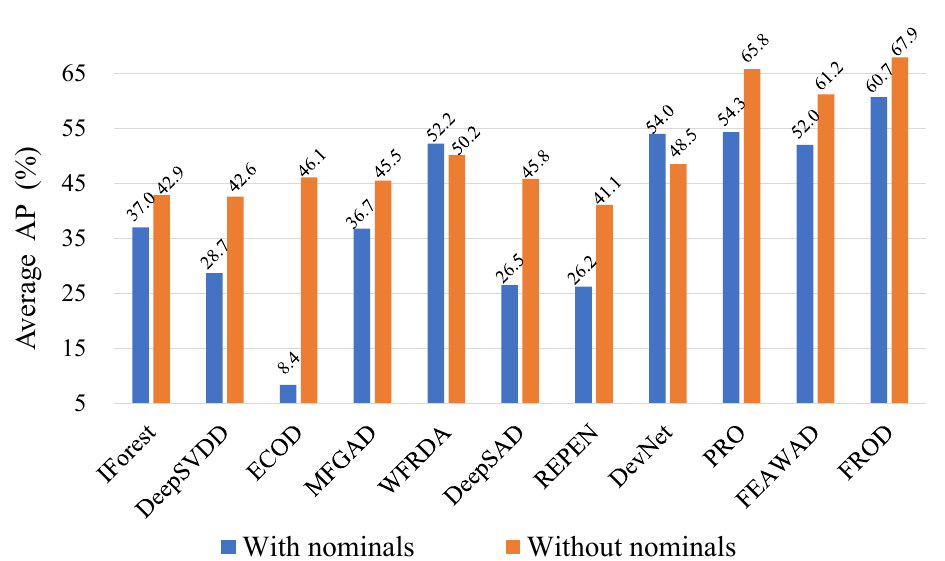}
    \caption{Average AP (\%) on datasets with and without nominal attributes (@1\% labeled data)}\label{AP-nominals}
    \end{figure}
    
\subsubsection{Effectiveness with mixed-attribute data} 
    As heterogeneous data with various types of values are widely available in real applications, this part studies the detection performance on datasets with and without nominal attributes. The four datasets Arrhythmia, Mushroom1, Mushroom2, and Sick contain nominal attributes, while others are all numerical. Fig. \ref{AUC-nominals} summarizes the average AUC scores (@1\% labeled data). One can observe that FROD performs 9.1\% and 3.7\% higher than the best semi-supervised competing model REPEN and the best-unsupervised algorithm WFRDA on datasets with nominal attributes, respectively. Fig. \ref{AP-nominals} shows the results in terms of AP. FROD achieves a 6.4\% higher AP score than the best competitor PReNet, and 8.5\% higher than the best-unsupervised competing algorithm WFRDA. 
    This may be attributed to FROD's utilization of a hybrid fuzzy membership function to directly handle nominal attributes without introducing improper assumptions in data. Therefore, FROD can make better use of heterogeneous information in data.
    
\subsubsection{Performances with various amounts of labeled data}
    This part tries to answer how effective FROD is in utilizing labeled data w.r.t various levels of supervision. The ratio of labeled data varies from 1\% to 16\% with the number of outliers ranging from 1 to 92.
    Fig. \ref{fig_train_ave} shows the average AUC and AP scores w.r.t multiple proportions of labeled data, and Fig. \ref{fig_train_auc} details the AUC results on 16 datasets. 

\begin{figure}[!hbt]
    \centering
    \includegraphics[width=9cm]{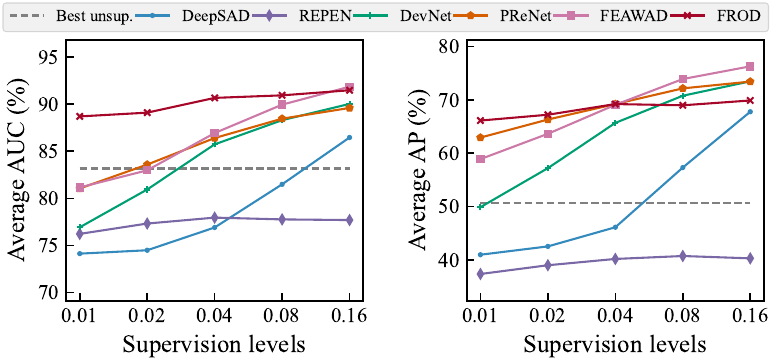}
    \caption{Average AUC and AP scores on 16 datasets w.r.t multiple supervision levels. The best-unsupervised method is WFRDA, which achieves the best average AUC among its counterparts.} \label{fig_train_ave}
    \end{figure}

\begin{figure*}[!hbt]
    \centering
    \includegraphics[width=\textwidth]{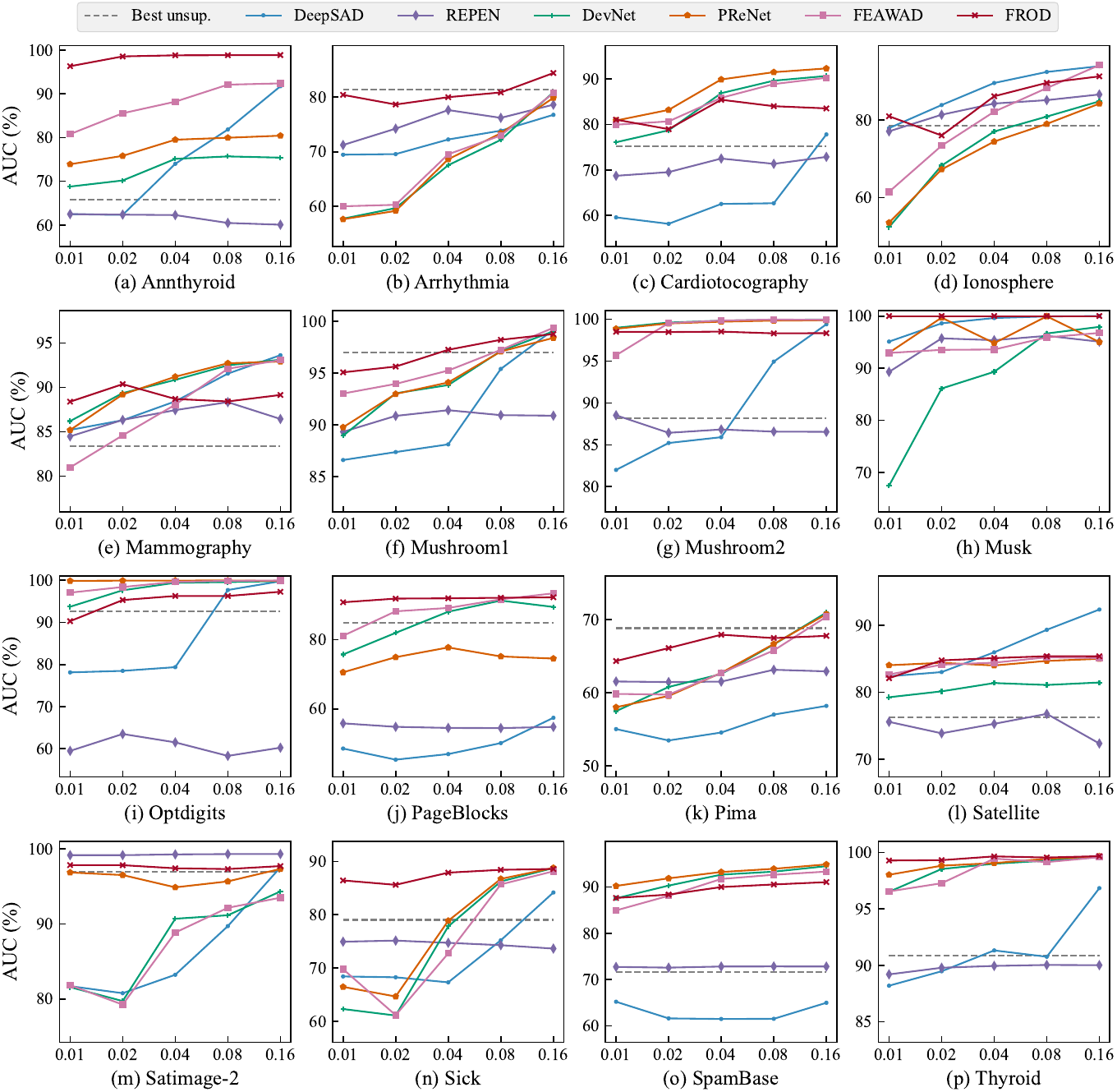}
    \caption{AUC scores w.r.t multiple proportions of labeled data on 16 experimental datasets. The best unsupervised method is WFRDA.}
    \label{fig_train_auc}
    \end{figure*}
    
    As seen from Fig. \ref{fig_train_ave}, FROD is the most data-efficient method, which achieves the best average AUC and AP scores on 16 datasets with the lowest number of labeled outliers. FROD requires only 1/8 (1/4) of labeled data to achieve comparable AUC (AP) results to the competing method FEAWAD.
    Additionally, when a minimal number of labeled data (e.g., 2\% for AUC, 1\% for AP) are used, the semi-supervised methods outperform the best-unsupervised method WFRDA. For example, the average improvement of PReNet and FEAWAD over WFRDA is more than 12.25\% and 8.21\% in terms of AP, respectively. The result of FROD in the same condition is 15.44\% higher than WFRDA.
    FROD's data efficiency is mainly due to its adoption of the fuzzy relative entropy-based outlier scoring method from a considerable number of unlabeled data. This allows FROD to characterize outliers effectively and efficiently.
    
    Moreover, as shown in Fig. \ref{fig_train_auc}, although semi-supervised methods generally improve with the increase of supervision ratio, some competing deep learning-based detectors drop, e.g., REPEN on Annthyroid and Mushroom2, Optdigits and satelite; PReNet on PageBlocks and satimage-2; DeepSAD on SpamBase; as well as FROD on Mammography. 
    The different distribution of outliers may cause this, and they may reduce detection performance when the added data have abnormal behavior that conflicts with the previous. 
    In this case, FROD is more stable.

\subsubsection{Statistical significance analysis}
    To investigate whether there are statistically significant differences between FROD and comparison methods, we use Friedman’s test and Nemenyi’s post-hoc test \cite{Yuan2023WFRDA} to examine the statistical significance of the detection performances across 16 datasets.
    Nemenyi’s test produces a critical difference diagram that shows the critical difference (CD) value and the groups of comparison algorithms. The CD value is the minimum difference between two classifiers that are statistically significant. If the difference between two detectors is greater than CD, then they are considered significantly different from each other at a given significance level.

\begin{figure}[!h]
    \centering
    \includegraphics[width=7cm]{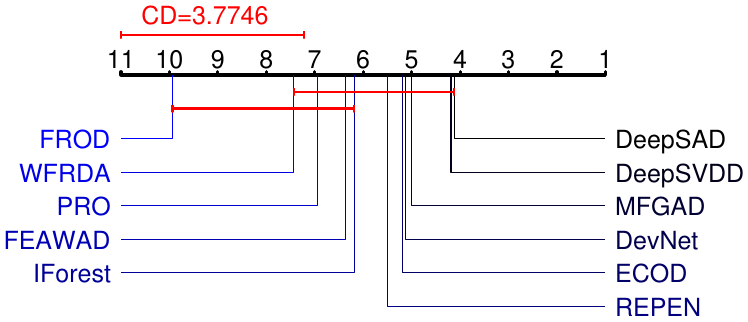}
    \caption{Nemenyi's test figure on AUC. The CD value is $3.7746$ at the significance level of $0.05$.}\label{AUC-CD}
    \end{figure}

    
    Fig. \ref{AUC-CD} shows the results of Nemenyi’s test in terms of AUC.   We can see that FROD is only connected by a red line with WFRDA, PReNet, FEAWAD, and IForest, while not connected with the other 6 comparative methods, i.e., DeepSAD, DeepSVDD, MFGAD, DevNet, ECOD, and REPEN. This means that FROD is statistically significantly different from these 6 algorithms. However, all the other algorithms are linked by a red line, which indicates that there is no statistically significant difference among them. 
    This is due to the fact that the experimental datasets we adopted originate from a very wide domain with complex and diverse anomaly mechanisms, and there is no one-size-fits-all detection method that can significantly outperform other methods in all cases \cite{han2022ADbench}.

\section{Conclusion}
    In this paper, we propose a fuzzy rough sets-based outlier detection (FROD) algorithm for mixed-attribute data in a semi supervised framework. We introduce attribute classification accuracy derived from FRS to evaluate the contribution of attribute sets in detecting outliers. Additionally, we define fuzzy relative entropy to characterize outliers from the perspective of uncertainty. Our FRS-based detection algorithm adeptly handles the uncertainty and imprecision inherent in complex data. 
    Compared to other detection algorithms intrinsically designed for numerical data, our approach can detect outliers with high precision and data efficiency on real-world datasets with mixed attributes.
    In future work, we will take into account the types of outliers as well as explore the boundary regions of FRS to improve detection performance and robustness.

\section{Acknowledgement}
This work was supported by the National Natural Science Foundation of China (62306196 and 62376230), the Sichuan Science and Technology Program (2024NSFTD0049, 2024YFHZ0144, 2024YFHZ0089 and 2024NSFSC0443), and the Fundamental Research Funds for the Central Universities (YJ202245).

\bibliographystyle{elsarticle-num-names}
\bibliography{mybib}

\end{document}